\documentclass[11pt]{article}

\usepackage[top=1.25in, bottom=1.25in, left=1.25in, right=1.25in]{geometry}
\usepackage[utf8]{inputenc} % allow utf-8 input
\usepackage[T1]{fontenc}    % use 8-bit T1 fonts
\usepackage{hyperref}       % hyperlinks
\usepackage{url}            % simple URL typesetting
\usepackage{booktabs}       % professional-quality tables
\usepackage{amsfonts}       % blackboard math symbols
\usepackage{nicefrac}       % compact symbols for 1/2, etc.
\usepackage{microtype}      % microtypography
\usepackage{tcolorbox}

\usepackage{array}
\newcolumntype{I}{!{\vrule width 3pt}}
\newlength\savedwidth

\newlength\savewidth

\usepackage{amsthm,amsmath,amssymb}
\usepackage{color,graphicx}
\usepackage{multicol}
\usepackage{multirow}
\usepackage{caption}
\usepackage{subcaption}
\usepackage[boxruled,linesnumbered]{algorithm2e}
\usepackage{algorithmic}
\usepackage{enumitem}
\usepackage{url}
\usepackage{hhline}

\renewcommand{\epsilon}{\varepsilon}

\newcommand{\Exp}{\mathbb{E}}

\newtheorem{theorem}{Theorem}[section]
\newtheorem{lemma}[theorem]{Lemma}
\newtheorem{definition}[theorem]{Definition}
\newtheorem{corollary}[theorem]{Corollary}
\newtheorem{claim}[theorem]{Claim}

\newtheorem{remark}[theorem]{Remark}

\newcommand{\I}{\mathrm{I}}

\newcommand{\eat}[1]{}

\newcommand{\R}{\mathbb{R}}

\newcommand{\calA}{\mathcal{A}}

\newcommand{\calD}{\mathcal{D}}

\newcommand{\calF}{\mathcal{F}}

\newcommand{\calX}{\mathcal{X}}

\title{\bf Stable and Fair Classification}
\author{Lingxiao Huang\thanks{EPFL, Switzerland.} and Nisheeth K. Vishnoi\thanks{Yale University, USA. Email: nisheeth.vishnoi@yale.edu}}

\date{}

\allowdisplaybreaks[4]

\begin{document}
	
	\maketitle

	\begin{abstract}
		Fair classification has been a topic of intense study in machine learning, and several algorithms have been proposed towards this important task.
		However, in a recent study, Friedler et al. observed that fair classification algorithms may not be  stable  with respect to variations in the training dataset -- a crucial consideration in several real-world applications.
		Motivated by their work, we study the problem of designing classification algorithms that are both fair and stable.
		We propose an extended framework based on fair classification algorithms that are formulated as optimization problems, by introducing a stability-focused regularization term.
		Theoretically, we prove a stability guarantee, that was lacking in fair classification algorithms, and also provide an accuracy guarantee for our extended framework.
		Our accuracy guarantee can be used to inform the selection of the regularization parameter in our framework.
		To the best of our knowledge, this is the first work that combines stability and fairness in automated decision-making tasks.
		We assess the benefits of our approach empirically by extending several fair classification algorithms that are shown to achieve the best balance between fairness and accuracy over the \textbf{Adult} dataset.
		Our empirical results show that our  framework indeed improves the stability at only a slight sacrifice in accuracy.
	\end{abstract}
	\newpage
	
	\tableofcontents

	\newpage
	
		\section{Introduction}
	\label{sec:intro}

	Fair classification has fast become a central problem in machine learning due to  concerns of bias with respect to  sensitive attributes in automated decision making, e.g., against African-Americans while predicting future criminals~\cite{flores2016false,angwin2016machine,berk2009role}, granting loans~\cite{dedman1988color}, or NYPD stop-and-frisk~\cite{goel2016precinct}.
	% and against women while recommending jobs~\cite{datta2015automated}.
	%
	Consequently, a host of fair classification algorithms have been proposed; see~\cite{bellamy2018ai}.
	%  
	
	% Due to practical needs, a natural question arises -- whether existing fair classification algorithms are stable.
	%
	In a recent study, \cite{friedler2018comparative} pointed out  that several existing fair classification algorithms are not ``stable''.
	% -- the fairness of the classifier depended significantly on the training set. 
	%
	In particular, they considered the standard deviation of a fairness metric (statistical rate, that measures the discrepancy between the positive proportions of two groups; see Eq.~\eqref{eq:gamma}) and accuracy over ten random training-testing splits with respect to race/sex attribute over the \textbf{Adult} dataset.
	%
	% For instance, their empirical results indicated that \textbf{ZVRG}~\cite{zafar2017fairness} is potentially the best choice in a balance between the fairness metric and accuracy. 
	%
	They observed that the standard deviation of the fairness metric is 2.4$\%$ for the algorithm in \cite{kamishima2012fairness} (\textbf{KAAS}) with respect to the race attribute, and is 4.1$\%$ for that in \cite{zafar2017fairness} (\textbf{ZVRG}) with respect to the sex attribute.
	These significant standard deviations imply that the classifier learnt from the respective fair classification algorithms might perform differently depending on the training dataset.
	%
	
	%	Motivated by their work, we investigate the problem of designing fair classification algorithms that are also predicted stable, i.e., random training sets only result in small changes on the learned classifiers. 
	%
	
	Stability is a crucial consideration in  classification~\cite{bousquet2002stability,mukherjee2006learning,briand2009similarity,fawzi2018analysis}, and has been investigated in several real-world applications, e.g., advice-giving agents~\cite{gershoff2003consumer,van2005factors}, recommendation systems~\cite{adomavicius2012stability,adomavicius2011maximizing,adomavicius2016classification}, and judicial decision-making~\cite{shapiro1965stability}. 
	Stable classification algorithms can also provide defense for data poisoning attacks, whereby adversaries want to corrupt the learned model by injecting false training data~\cite{biggio2012poisoning,mei2015using,steinhardt2017certified}.

	There is a growing number of scenarios in which stable {\em and} fair classification algorithms are desired.
	One example is recommendation systems that rely on classification algorithms~\cite{park2012literature,portugal2018the}.
	Fairness is often desired  in recommendation systems, e.g., to check  gender inequality in recommending high-paying jobs~\cite{farahat2012effective,datta2015automated,sweeney2013discrimination}.
	Moreover, stability is also important for the reliability and acceptability of recommendation systems~\cite{adomavicius2012stability,adomavicius2011maximizing,adomavicius2016classification}.
	Another example is that of a judicial decision-making system, in which fair classification algorithms are being deployed to avoid human biases for specific sensitive attributes, e.g., against African-Americans~\cite{flores2016false,angwin2016machine,berk2009role}.
	The dataset, that incorporates collected personal information, may be noisy due to  measurement errors, privacy issues, or even data poisoning attacks~\cite{lam2004shilling,mobasher2007toward,o2004evaluation,barreno2010security} and, hence, it is desirable that the fair classifier also be stable against perturbations in the dataset. 
	%
	%	Otherwise, people can suspect that the guilt decision is problematic and incredible.
	%
	%	Hence, we would like to achieve a judicial decision-making system that is both stable and fair.
	%   
	% The importance of stability against noises also reflects in health-care and granting loans.
	%

	\subsection{Our contributions}
	\label{subsec:contribution}
	In this paper, we initiate a study of stable and fair classifiers in automated decision-making tasks.
	In particular,	we consider the class of  fair classification algorithms that are formulated as  optimization problems that minimize the empirical risk while being constrained to being fair.
	The collection $\calF$ of possible classifiers is assumed to be a reproducing kernel Hilbert space (RKHS) (see Program~\eqref{eq:progcon} for a definition); this includes many recent fair classifiers such as ~\cite{zafar2017fair,zafar2017fairness,goel2018non}. 
	Our main contribution is an algorithmic framework that incorporates the notion of uniform stability \cite{bousquet2002stability} -- the maximum $l_{\infty}$-distance between the risks of two classifiers learned from two training sets that differ in a single sample (see Definition~\ref{def:class_stability}).
	This allows us to address the stability issue observed by \cite{friedler2018comparative}.
	To achieve uniform stability, we introduce a stability-focused regularization term to the objective function of fair classifier  (Program~\eqref{eq:progstable}), which is motivated by the work of~\cite{bousquet2002stability}.
	Although some existing fair classification algorithms~\cite{kamishima2012fairness,goel2018non} use regularizers, they do not seem to realize that (and show how) the regularization term can also make the algorithm more stable.
	Under mild assumptions on the loss function (Definition~\ref{def:admissible}), we prove that our extended framework indeed has an additional uniform stability guarantee $\tilde{O}(\frac{1}{\lambda N})$, where $\lambda$ is the regularization parameter and $N$ is the size of the training set (Theorems~\ref{thm:stable}).
	Moreover, if $\calF$ is a linear model, we can achieve a slightly better stability guarantee (Theorem~\ref{thm:stable2}).
	%
	%	To the best of our knowledge, this is the first work that combines stability and fairness in automated decision-making tasks.
	%
	Our stability guarantee also implies an empirical risk guarantee that can be used to inform the selection of the regularization parameter in our framework. 
	By letting $\lambda = \Theta(\frac{1}{\sqrt{N}})$, the increase in the empirical risk by introducing the regularization term can be bounded by $\tilde{O}(\frac{1}{\sqrt{N}})$ (Theorems~\ref{thm:stable} and~\ref{thm:stable2}, Remark~\ref{remark:assumption}).
	As a consequence, our stability guarantee also implies a generalization bound -- the expected difference between the expected risk and the empirical risk is $\tilde{O}(\frac{1}{\lambda N})$ (Corollaries~\ref{cor:gen} and~\ref{cor:gen2}).
	%
	% We also analyze our framework in several commonly used settings, including soft margin SVM, least square regression and logistic regression (Corollaries~\ref{cor:SVM}-\ref{cor:logisitc}).
	%
	
	Further, we conduct an empirical evaluation over the \textbf{Adult} dataset and apply our framework to several fair classification algorithms, including \textbf{KAAS}~\cite{kamishima2012fairness},  \textbf{ZVRG}~\cite{zafar2017fair} and \textbf{GYF}~\cite{goel2018non} (Section~\ref{sec:experiment}).
	Similar to~\cite{friedler2018comparative}, we evaluate the fairness metric and accuracy of these algorithms and our extended algorithms.
	Besides, we also compute the expected number of different predictions over the test dataset between classifiers learned from two random training sets as a stability measure $\mathrm{stab}$ (Eq.~\eqref{eq:stability}).
	The empirical results show that our classification algorithms indeed achieve better stability guarantee, while being fair. 
	For instance, with respect to the sex attribute, the standard deviation of the fairness metric of \textbf{ZVRG} improves from 4.1$\%$ (\cite{friedler2018comparative}) to about 1$\%$ using our extended algorithm, and the stability measure $\mathrm{stab}$ decreases from 70 ($\lambda=0$) to 25 ($\lambda=0.02$).
	Meanwhile, the loss in accuracy due to imposing stability-focused regularization term is small (at most 1.5$\%$).

	Overall, we provide the first extended framework for stable and fair classification, which makes it flexible and easy to use, slightly sacrifices accuracy, and performs well in practice.

	\subsection{Other related work}
	\label{subsec:related}
	
	%	In recent years, there are increasingly many works on fair classification in different contexts.	
	%
	From a technical view,  most relevant prior works formulated the fair classification problem as a constrained optimization problem, e.g., constrained to statistical parity~\cite{zafar2017fairness,menon2018the,goel2018non,celis2018classification}, or equalized odds~\cite{hardt2016equality,zafar2017fair,menon2018the,celis2018classification}.
	Our extended framework can be applied to this type of fair classification.
	Another approach for fair classification is to shift the decision boundary of a baseline classifier, e.g.,~\cite{fish2016confidence,hardt2016equality,goh2016satisfying,pleiss2017on, woodworth2017learning,dwork2018decoupled}.
	Finally, a different line of research pre-processes the training data with the goal of removing the bias for learning, e.g.,~\cite{kamiran2009classifying,luong2011k,kamiran2012data,zemel2013learning,feldman2015certifying,krasanakis2018adaptive}.
	%
	
	%A line of research for stability is to investigate whether learning algorithms are stable.
	%
	Several prior works~\cite{bousquet2002stability,shalev2010learnability,maurer2017second,meng2017generalization} study the stability property for empirical risk minimization.
	\cite{hardt2016train}, \cite{london2016generalization}
	and \cite{kuzborskij2018data} showed that the stochastic gradient descent method is stable.
	Moreover, several recent works studied stability in deep neural networks~\cite{raghu2017expressive,vidal2017mathematics}.
	Stability has been investigated in other automated decision-making tasks, e.g., feature selection~\cite{nogueira2018on} and structured prediction~\cite{london2013collective,london2014pac,london2016stability}.

	There exists a related notion to stability, called differential privacy, where the prediction for any sample should not change with high probability if the training set varies a single element. 
	By \cite{wang2016learning}, differential privacy implies a certain stability guarantee. 
	Hence, it is possible to achieve stable and fair classifiers by designing algorithms that satisfy differential privacy and fairness simultaneously.
	Recent studies~\cite{hajian2016algorithmic,hajian2015discrimination,kashid2015discrimination,kashid2017discrimination,ruggieri2014anti} have expanded the application of methods to achieve both goals; see a recent paper~\cite{ekstrand2018privacy} for more discussions.
	However, these methods are almost all heuristic and without theoretical guarantee.
	There also remains the open problem of characterizing under what circumstances and definitions, privacy and fairness are simultaneously achievable, and when they compete with each other.

	\section{Our Model}
	\label{sec:pre}
	
	\subsection{Preliminaries}
	We consider the Bayesian model for classification. 
	Let $\Im$ denote a joint distribution over the domain $\calD=\calX\times [p]\times\left\{-1,1\right\}$ where $\calX$ is the feature space.
	Each sample $(X,Z,Y)$ is drawn from $\Im$ where $Z\in [p]$ represents a sensitive attribute,\footnote{Our results can be generalized to multiple sensitive attributes $Z_1,\ldots,Z_m$ where $Z_i\in [p_i]$. We omit the details.} and $Y\in \left\{-1,1\right\}$ is the label of $(X,Z)$ that we want to predict.

	Let $\calF$ denote the collection of all possible classifiers $f:\calX\rightarrow \mathbb{R}$.
	Given a loss function $L(\cdot ,\cdot)$ that takes a classifier $f$ and a distribution $\Im$ as arguments, the goal of fair classification is to learn a classifier $f\in \calF$ that minimizes the expected risk 
	$
	R(f):=\Exp_{s\sim \Im} \left[L(f, s)\right].
	$
	However, since $\Im$ is often unknown, we usually use the empirical risk to estimate the expected risk~\cite{bousquet2002stability,shalev2010learnability,maurer2017second}, i.e., given a training set $S=\left\{s_i=\left(x_i,z_i,y_i\right)\right\}_{i\in [N]}$ (where $(x_i,z_i,y_i)\in \calD$), the objective is to learn a classifier $f\in \calF$ that minimizes the empirical risk 
	$E(f):=\frac{1}{N}\sum_{i\in [N]} L(f,s_i)$.
	Denote by $\Pr_{\Im}[\cdot]$ the probability with respect to $\Im$.
	If $\Im$ is clear from context, we simply denote $\Pr_{\Im}[\cdot]$ by $\Pr[\cdot]$.
	A fair classification algorithm $\calA$ can be considered as a mapping $\calA: \calD^*\rightarrow \calF$, which learns a classifier $\calA_S\in \calF$ from a training set $S\in \calD^*$.

	\subsection{Stability measure}
	\label{sec:uniform_stability}
	
	In this paper, we consider the following stability measure introduced by~\cite{bousquet2002stability}, which was also used by~\cite{shalev2010learnability,maurer2017second,meng2017generalization}. 
	This notion of stability measures whether the \emph{risk} of the learnt classifier is stable under  replacing one sample in the training dataset. 

	\begin{definition}[\textbf{Uniform stability~\cite{bousquet2002stability}}]
		\label{def:class_stability}
		Given an integer $N$, a real-valued classification algorithm $\calA$ is $\beta_N$-uniformly stable with respect to the loss function $L(\cdot,\cdot)$ if the following holds: for all $i\in [N]$ and $S, S^i\in \calD^N$,
		\begin{eqnarray*}
			\left\| L(\calA_S,\cdot)- L(\calA_{S^{i}},\cdot) \right\|_{\infty}  \leq \beta_N,
		\end{eqnarray*}
		i.e., for any training sets $S, S^i\in \calD^N$ that differ by the $i$-th sample, the $l_{\infty}$-distance between the risks of $\calA_S$ and $\calA_{S^{i}}$ is at most $\beta_N$.
	\end{definition}
	\noindent
	By definition, algorithm $\calA$ is stable if $\beta_N$ is small.

	Since classification algorithms usually minimize the empirical risk, it is easier to bound to provide theoretcial bounds on the risk difference.
	This is the reason  we consider the notion of uniform stability.
	Moreover, uniform stability might imply that the prediction variation is small with a slight perturbation on the training set. 
	Given an algorithm $\calA$ and a sample $x\in \calX$, we predict the label to be +1 if $\calA(x)\geq 0$ and to be -1 if $\calA(x)< 0$.
	In the following, we summarize the stability property considered in~\cite{friedler2018comparative}. 
	\begin{definition}[\textbf{Prediction stability}]
		\label{def:predict_stability}
		Given an integer $N$, a real-valued classification algorithm $\calA$ is $\beta_N$-prediction stable if the following holds: for all $i\in [N]$,
		\begin{eqnarray*}
			\Pr_{S, S^i\in \calD^N, X\sim \Im}\left[\I\left[\calA_S(X)\geq 0\right]\neq  \I\left[\calA_{S^i}(X)\geq 0\right] \right] \leq \beta_N,\footnotemark
		\end{eqnarray*}
		\footnotetext{Here, $\I\left[\cdot\right]$ is the indicator function.}
		i.e., given two training sets $S, S^i\in \calD^N$ that differ by the $i$-th sample, the probability that $\calA_S$ and $\calA_{S^i}$ predict differently is at most $\beta_N$.
	\end{definition}
	The following lemma shows that uniform stability implies prediction stability.

	\begin{lemma}[\textbf{Uniform stability implies prediction stability}]
		\label{thm:uniform_to_prediction}
		Given an interger $N$, if algorithm $\calA$ is $\beta_N$-uniformly stable with respect to the loss function $L(\cdot,\cdot)$ and the loss function satisfies that for any $f,f'\in \calF$, $s=(x,z,y)\in \calD$, 
		\[
		|f(x)-f'(x)|\leq \tau\cdot\left|L(f,s)-L(f',s)\right|,
		\]
		then the prediction stability $\calA$ is upper bounded by
		$\Pr_{S,\calA}\left[|\calA_S(X)|\leq \tau \beta_N\right]$.
	\end{lemma}
	\begin{proof}
		For any $S, S^i\in \calD^N$ and $s=(x,z,y)$, we have
		\begin{eqnarray*}
			\left|\calA_S(x)-\calA_{S^i}(x)\right|\leq \tau\cdot \left|L(\calA_S,s)-L(\calA_{S^i},s) \right| \leq \tau\beta_N.
		\end{eqnarray*}
		Hence, if $|\calA_S(\cdot)|> \tau \beta_N$, then we have
		\[
		\I\left[\calA_S(x)\geq 0\right]=  \I\left[\calA_{S^i}(x)\geq 0\right].
		\]
		By Definition~\ref{def:predict_stability}, this implies the lemma.
	\end{proof}

	\subsection{The stable and fair optimization problem}
	\label{sec:model}
	
	Our goal is to design fair classification algorithms that have a uniform stability guarantee.
	We focus on extending fair classification algorithms that are formulated as constrained empirical risk minimization problem over the collection  $\calF$ of classifiers that is a reproducing kernel Hilbert space (RKHS), e.g.,~\cite{zafar2017fair,zafar2017fairness,goel2018non}; see the following program.

	\begin{tcolorbox}
		\begin{equation} \tag{ConFair}
		\label{eq:progcon}
		\begin{split}
		& \min_{f\in \calF} \frac{1}{N}\sum_{i\in [N]} L(f,s_i) \quad s.t. \\
		& ~ \Omega_S(f)\leq 0.
		\end{split}
		\end{equation}
	\end{tcolorbox}
	
	\noindent
	Here, $\Omega_S(\cdot): \calF\rightarrow \R^a$ is a convex function given explicitly for a specific fairness requirement on the dataset $S$. 
	When $S$ is clear by the content, we may use $\Omega(\cdot)$ for simplicity.
	For instance, if we consider the statistical rate $\gamma(f)$ (Eq.~\eqref{eq:gamma}) as the fairness metric, then the fairness requirement can be $0.8-\gamma(f)\leq 0$.
	However, $0.8-\gamma(f)$ is non-convex with respect to $f$.
	To address this problem, in the literature, one usually defines a convex function $\Omega(f)$ to estimate $0.8-\gamma(f)$, e.g., $\Omega(f)$ is formulated as a covariance-type function which is the average signed distance from
	the feature vectors to the decision boundary in \cite{zafar2017fairness}, and is formulated as the weighted sum of the logs of the empirical estimate of favorable bias in \cite{goel2018non}.
	In what follows,	a fair classification algorithm is an algorithm that  solves Program~\eqref{eq:progcon}.

	Note that an empirical risk minimizer of Program~\eqref{eq:progcon} might heavily depend on and even overfit the training set. 
	Hence, replacing a sample from the training set might cause a significant change in the learnt fair classifier -- the uniform stability guarantee might be large.
	To address this problem, a useful high-level idea is to introduce a regularization term to the objective function, which can penalize the
	``complexity'' of the learned classifier.
	Intuitively, this can make the change in the learnt classifier smaller when a sample from the training set is replaced.
	This idea comes from \cite{bousquet2002stability} who considered stability for unconstrained empirical risk minimization.

	Motivated by the above intuition, we consider the following constrained optimization problem which is an extension of Program~\eqref{eq:progcon} by introducing a stability-focused regularization term $\lambda\|f\|_k^2$. 
	Here, $\lambda>0$ is a regularization parameter and $\|f\|_k^2$ is the norm of $f$ in RKHS $\calF$ where $k$ is the kernel function (defined later in Definition~\ref{def:RKHS}).
	We consider such a regularization term since it satisfies a nice property  that relates $|f(x)|$ and $\|f\|_k$ for any $x\in \calX$ (Claim~\ref{cl:RKHS}).
	This property is useful for proving making the intuition above concrete and providing a uniform stability guarantee.

	\begin{tcolorbox}
		\begin{equation} \tag{Stable-Fair}
		\label{eq:progstable}
		\begin{split}
		& \min_{f\in \calF} \frac{1}{N} \sum_{i\in [N]} L(f,s_i) + \lambda \|f\|_k^2 \quad s.t. \\
		& ~ \Omega(f)\leq 0.
		\end{split}
		\end{equation}
	\end{tcolorbox}

	\noindent
	Our extended algorithm $\calA$ is to compute a minimizer $\calA_S$ of Program~\eqref{eq:progstable} by classic methods, e.g., stochastic gradient descent~\cite{boyd2008stochastic}.

	\begin{remark}
		\label{remark:confair}
		We first discuss the motivation of considering fair classification algorithms formulated as Program~\eqref{eq:progcon}. 
		The main reason is that such algorithms can achieve a good balance between fairness and accuracy, but might not be stable.
		For instance, \cite{friedler2018comparative} observed that \textbf{ZVRG}~\cite{zafar2017fairness} achieves the best balance between fairness and accuracy with respect to race/sex attribute over the \textbf{Adult} dataset.
		However, as mentioned in Section~\ref{sec:intro}, \textbf{ZVRG} is not stable depending on a random training set. 
		Hence, we would like to improve the stability of \textbf{ZVRG} while keeping its balance between fairness and accuracy.
		Note that our extended framework can incorporate  multiple sensitive attributes if the fairness constraint $\Omega(f)\leq 0$ deals with multiple sensitive attributes, e.g.,~\cite{zafar2017fair,zafar2017fairness,goel2018non}.
	\end{remark}
	
	\noindent
	It remains to define the regularization term $\|f\|_k$ in RKHS.

	\begin{definition}[\textbf{Regularization in RKHS}]
		\label{def:RKHS}
		We call $T(\cdot): \calF\rightarrow \R_{\geq 0}$ a regularization term in an RKHS $\calF$ if, for any $f\in \calF$,
		$T(f):=\|f\|_k^2$,
		where $k$ is a kernel function satisfying that 
		1) $\left\{k(x,\cdot):x\in \calX \right\}$ is a span of $\calF$; 2) for any $x\in \calX$ and $f\in \calF$, $f(x)=\langle f,k(x,\cdot) \rangle$.
	\end{definition}
	\noindent
	Given a training set $S=\left\{s_i=\left(x_i,z_i,y_i\right)\right\}_{i\in [N]}$ and a kernel function $k: S\times S\rightarrow \R$, by definition, each classifier is a vector space by linear combinations of $k(x_i,\cdot)$, i.e.,
	$f(\cdot) = \sum_{i\in [N]} \alpha_i k(x_i,\cdot)$.
	Then for any $x\in \calX$, 
	\begin{eqnarray}
	\label{eq:kernel}
	f(x) = \langle  \sum_{i\in [N]} \alpha_i k(x_i,\cdot), k(x,\cdot) \rangle =  \sum_{i\in [N]} \alpha_i k(x_i,x).
	\end{eqnarray}
	For instance, if $k(x,y)=\langle x,y \rangle$, then each classifier $f$ can be represented by 
	\begin{eqnarray*}
		f(x) &\stackrel{\eqref{eq:kernel}}{=}& \sum_{i\in [N]} \alpha_i k(x_i,x) = \sum_{i\in [N]} \alpha_i \langle x_i,x\rangle  =\langle \sum_{i\in [N]} \alpha_i x_i,x\rangle = \langle \beta, x \rangle,
	\end{eqnarray*}
	where $\beta = \sum_{i\in [N]} \alpha_i x_i$.
	Thus, by the Cauchy-Schwarz inequality, we have the following useful property.

	\begin{claim} (\cite{bousquet2002stability})
		\label{cl:RKHS}
		Suppose $\calF$ is a RKHS with a kernel $k$.
		For any $f\in \calF$ and any $x\in \calX$, we have
		$|f(x)|\leq \|f\|_k \sqrt{k(x,x)}$. 
	\end{claim}
	
	\begin{remark}
		There exists another class of fair classification algorithms, which introduce a fairness-focused regularization term $\mu\cdot\Omega(\cdot)$ to the objective function; see the following program. 
		\begin{tcolorbox}
			\begin{equation} \tag{RegFair}
			\label{eq:progreg}
			\min_{f\in \calF} \frac{1}{N}\sum_{i\in [N]} L(f,s_i) + \mu\cdot \Omega(f).
			\end{equation}
		\end{tcolorbox}
		\noindent
		This approach is applied in several prior work, e.g.,~\cite{kamishima2012fairness,davies2017algorithmic,goel2018non}. 
		We can also extend this program by introducing a stability-focused regularization term $\lambda\|f\|_k^2$.
		\begin{tcolorbox}
			\begin{equation*}
			\label{eq:progreg_fair}
			\min_{f\in \calF} \frac{1}{N}\sum_{i\in [N]} L(f,s_i) + \mu\cdot \Omega(f) + \lambda \|f\|_k^2.
			\end{equation*}
		\end{tcolorbox}

		\noindent
		By Lagrangian principle, there exists a value $\mu\geq 0$ such that Program~\eqref{eq:progreg} is equivalent to Program~\eqref{eq:progcon}.
		Thus, by solving the above program, we can obtain the same stability guarantee, empirical risk guarantee and generalization bound as for Program~\eqref{eq:progstable}.
	\end{remark}

	\section{Theoretical Results}
	\label{sec:alg}
	
	In this section, we analyze the performance of algorithm $\calA$ that solves Program~\eqref{eq:progstable} (Theorem~\ref{thm:stable}).
	Moreover, if $\calF$ is a linear model, we can achieve a slightly better stability guarantee (Theorem~\ref{thm:stable2}).
	%analyze the performance of our extended algorithm that solves Program
	
	Given a training set $S=\left\{s_i=\left(x_i,z_i,y_i\right)\right\}_{i\in [N]}$,
	by replacing the $i$-th element from $S$, we denote
	\[
	S^{i}:=\left\{s_1,\ldots,s_{i-1},s'_{i},s_{i+1},\ldots,s_N\right\}.
	\]
	Before analyzing the performance of algorithm $\calA$, we give the following definition for a  loss function.
	\begin{definition}[\textbf{$\sigma$-admissible~\cite{bousquet2002stability}}]
		\label{def:admissible}
		The loss function $L(\cdot,\cdot)$ is called $\sigma$-admissible with respect to $\calF$ if for any $f\in \calF$, $x,x'\in \calX$ and $y\in \left\{-1,1\right\}$,
		\[
		\left|L(f(x),y)-L(f(x'),y) \right| \leq \sigma \left|f(x)-f(x') \right|.
		\]
	\end{definition}
	\noindent
	By definition, $L(\cdot,\cdot)$ is $\sigma$-admissible if $L(f,y)$ is $\sigma$-Lipschitz with respect to $f$.

	\subsection{Main theorem for Program~\eqref{eq:progstable}}
	\label{sec:main}
	
	Now we can state our main theorem which indicates that under reasonable assumptions of the loss function and the kernel function, algorithm $\calA$ is uniformly stable.

	\begin{theorem}[\textbf{Stability and empirical risk guarantee by solving Program~\eqref{eq:progstable}}]
		\label{thm:stable}
		Let $\calF$ be a RKHS with kernel $k$ such that $\forall x\in \calX$, $k(x,x)\leq \kappa^2<\infty$. 
		Let $L(\cdot,\cdot)$ be a $\sigma$-admissible differentiable convex function with respect to $\calF$.
		Suppose algorithm $\calA$ computes a minimizer $\calA_S$ of Program~\eqref{eq:progstable}.
		Then $\calA$ is $\frac{\sigma^2 \kappa^2}{\lambda N}$-uniformly stable.

		Moreover, denote $f^\star$ to be an optimal fair classifier that minimizes the expected risk and satisfies $\|f^\star\|_k\leq B$, i.e., 
		$f^\star:= \arg \min_{f\in \calF: \Omega(f)\leq 0} \Exp_{s\in \Im}\left[ L(f,s)\right]$. 
		We have
		\[
		\Exp_{S\sim \Im^N} \left[R(\calA_S) \right] - \Exp_{s\sim \Im} \left[L(f^\star,s)\right] \leq \frac{\sigma^2 \kappa^2}{\lambda N}+\lambda B^2.
		\]
	\end{theorem}

	\begin{remark}
		\label{remark:assumption}
		We show that the assumptions of Theorem~\ref{thm:stable} are reasonable.
		We first give some examples of $L(\cdot,\cdot)$ in which $\sigma$ is constant. 
		In the main body of the paper, we directly give the constant. 
		The details can be found in Appendix~\ref{sec:dis}.
		\begin{itemize}
		\item Prediction error: Suppose $f(x)\in \left\{-1,1\right\}$ for any pair $(f,x)$. Then $L(f(x),y) = \I\left[f(x)\neq y\right] $ is $\frac{1}{2}$-admissible.
		\item Soft margin SVM: $L(f,s) = (1-yf(x))_+$,\footnote{$(a)_+ = a$ if $a\geq 0$ and otherwise $(a)_+ = 0$.} is 1-admissible.
		\item Least Squares regression: Suppose $f(x)\in [-1,1]$ for any $x\in \calX$. 
		Then we have that $L(f,s) = (f(x)-y)^2$ is 4-admissible.
		\item Logistic regression: $L(f,s) = \ln (1+e^{-y f(x)})$ is 1-admissible.
		\end{itemize}
		\noindent
		Then we give examples of kernel $k$ in which $\kappa^2$ is constant.
		\begin{itemize}
		\item Linear: $k(x,y) = \langle x,y \rangle$. Then $k(x,x)=\|x\|_2^2$ and we can let $\kappa^2 = \max_{x\in \calX}\|x\|_2^2$.
		\item Gaussian RBF: $k(x,y)=e^{-\|x-y\|^2}$. Then we can let $\kappa^2=k(x,x)=1$.
		\item Multiquadric: $k(x,y)=\left(\|x-y\|^2+c^2\right)^{1/2}$ for some constant $c>0$. Then we can let $\kappa^2=k(x,x)=c$.
		\item Inverse Multiquadric: $k(x,y)=\left(\|x-y\|^2+c^2\right)^{-1/2}$ for some constant $c>0$. Then we can let $\kappa^2=k(x,x)=1/c$.
		\end{itemize}
	\end{remark}
	
	\begin{remark}
		The statement of Theorem 3.2 seems similar to Lemma 4.1 of~\cite{bousquet2002stability}, while the analysis should be different due to the additional fairness constraints. 
		The critical difference is that the gradient of the objective function of Program~\ref{eq:stable} might not be 0 at the optimal point any more. 
		Thus, we need to develop a new analysis by applying the convexity of $\Omega(f)$. 

		Theorem~\ref{thm:stable} can be used to inform the selection of the regularization parameter $\lambda$. 
		On the one hand, the stability guarantee is tighter as $\lambda$ increases.
		On the other hand, the bound for the increase of the empirical risk contains a term $\lambda B^2$ and hence $\lambda$ should not increase to infinity.
		Hence, there exists a balance between achieving stability guarantee and utility guarantee.
		For instance, to minimize the increase of the empirical risk, we can set $\lambda = \frac{\sigma \kappa}{B \sqrt{N}}$. 
		Then the stability guarantee is upper bounded by $\frac{\sigma \kappa B}{\sqrt{N}}$ and the increase of the empirical risk is upper bounded by $\frac{2\sigma \kappa B}{\sqrt{N}}$.
	\end{remark}
	
	\noindent
	The generalization bound, i.e., the quality of the estimation $|R(\calA_S)-E(\calA_S)|$, depends on the number of samples $N$ and algorithm $\calA$, and has been well studied in the literature~\cite{adomavicius2011maximizing,bousquet2002stability,wainwright2006estimating,london2016stability}.
	Existing literature~\cite{bousquet2002stability,feldman2018generalization} claimed that uniform stability implies a generalization bound.
	We have the following corollary.

	\begin{corollary}[\bf Generalization bound from Theorem~\ref{thm:stable}]
		\label{cor:gen}
		Let $\calA$ denote the $\frac{\sigma^2 \kappa^2}{\lambda N}$-uniformly stable algorithm as stated in Theorem~\ref{thm:stable}.
		%
		%Suppose $0\leq L(\calA_S,s)\leq M$ for any $s\in \calD$, 
		We have
		\begin{enumerate}
			\item $\Exp_{S\sim \Im^N}\left[R(\calA_S)-E(\calA_S)\right]\leq \frac{\sigma^2 \kappa^2}{\lambda N}$.
			\item Suppose $S$ is a random draw of size $N$ from $\Im$. 
			With probability at least $1-\delta$, 
			\[
			R(\calA_S)\leq E(\calA_S) + 8\sqrt{\left(\frac{2\sigma^2 \kappa^2}{\lambda N} +\frac{1}{N} \right) \cdot\ln (8/\delta)}.
			\]
		\end{enumerate}
	\end{corollary}
	
	\begin{proof}
		The first generalization bound is directly implies by Lemma 7 of~\cite{bousquet2002stability}.
		The second generalization bound is a direct corollary of Theorem 1.2 of~\cite{feldman2018generalization}.
	\end{proof}

	\subsection{Better stability guarantee for linear models}
	\label{subsec:euclidean}
	
	In this section, we consider the case that $k(x,y)=\langle \phi(x),\phi(y) \rangle$ where $\phi: \calX\rightarrow \R^d$ is a feature map.
	It implies that $f(x) = \alpha^\top \phi(x)$ for some $\alpha\in \R^d$, i.e., $\calF$ is the family of all linear functions.
	In this case, we provide a stronger stability guarantee by the following theorem.

	\begin{theorem}[\textbf{Stability and utility guarantee for linear models}]
		\label{thm:stable2}
		Let $\calF$ be the family of all linear classifiers $\left\{f=\alpha^\top \phi(\cdot)\mid \alpha\in \R^d \right\}$. 
		Let 
		\[
		G = \sup_{f=\alpha^\top \phi(\cdot)\in \calF: \Omega(f)\leq 0} \sup_{s\in \calD} \|\nabla_\alpha L(f,s)\|_2.
		\]
		Suppose algorithm $\calA$ computes a minimizer $\calA_S$ of Program~\eqref{eq:progstable}.
		Then $\calA$ is $\frac{G^2}{\lambda N}$-uniformly stable.

		Moreover, denote $f^\star$ to be an optimal fair classifier that minimizes the expected risk and satisfies $\|f^\star\|_k\leq B$, i.e., 
		$f^\star:= \arg \min_{f\in \calF: \Omega(f)\leq 0} \Exp_{s\in \Im}\left[ L(f,s)\right]$.
		We have
		\[
		\Exp_{S\sim \Im^N} \left[R(\calA_S) \right] - \Exp_{s\sim \Im} \left[L(f^\star,s)\right] \leq \frac{G^2}{\lambda N}+\lambda B^2.
		\]
	\end{theorem}

	\noindent
	Note that we only have an assumption for the gradient of the loss function.
	Given a sample $s=(x,z,y)\in \calD$ such that $G = \sup_{f\in \calF: \Omega(f)\leq 0} \|\nabla_\alpha L(f,s)\|_2$, we have
	\[
	G=\|\nabla_\alpha L(f,s)\|_2 = \|\nabla_f L(f,s) \cdot \phi(x)\|_2.
	\]
	Under the assumption of Theorem~\ref{thm:stable}, we have 1) $|\nabla_f L(f,s)|\leq \sigma$ since $L(\cdot,\cdot)$ is $\sigma$-admissible with respect to $\calF$; 2) $\|\phi(x)\|_2=\sqrt{k(x,x)}\leq \kappa$.
	Hence, $G\leq \sigma \kappa$ which implies that Theorem~\ref{thm:stable} is stronger than Theorem~\ref{thm:stable2} for linear models.
	The proof idea is similar to that of Theorem~\ref{thm:stable} and hence we defer the proof to Appendix~\ref{sec:proof2}.
	Moreover, we directly have the following corollary similar to Corollary~\ref{cor:gen}.

	\begin{corollary}[\bf Generalization bound by Theorem~\ref{thm:stable2}]
		\label{cor:gen2}
		Let $\calA$ denote the $\frac{G^2}{\lambda N}$-uniformly stable algorithm as stated in Theorem~\ref{thm:stable2}.
		We have
		\begin{enumerate}
			\item $\Exp_{S\sim \Im^N}\left[R(\calA_S)-E(\calA_S)\right]\leq \frac{G^2}{\lambda N}$.
			\item Suppose $S$ is a random draw of size $N$ from $\Im$. 
			With probability at least $1-\delta$, 
			\[
			R(\calA_S)\leq E(\calA_S) + 8\sqrt{\left(\frac{2G^2}{\lambda N} +\frac{1}{N} \right) \cdot\ln (8/\delta)}.
			\]
		\end{enumerate}
	\end{corollary}

	\subsection{Analysis of Our Framework in Specified Settings}
	\label{sec:specified}

	Next, we show the stability guarantee of our framework in several specified models. % 
	We mainly analyze three commonly-used models: soft margin SVMs, least squares regression, and logistic regression.

	\paragraph{Soft margin SVMs.}
	Recall that $S=\left\{s_i=\left(x_i,z_i,y_i\right)\right\}_{i\in [N]}$ is the given training set.
	We first have a kernel function $k(\cdot,\cdot)$ that defines values $k(x_i,x_j)$. 
	Then each classifier $f$ is a linear combination of $k(x_i,\cdot)$, i.e.,
	\[
	f(\cdot) = \sum_{i\in [N]} \alpha_i k(x_i,\cdot)
	\]
	for some $\alpha\in \R^N$
	In the soft margin SVM model, we consider the following loss function
	\[
	L(f,s) = (1-yf(x))_+
	\]
	which is 1-admissible.
	Then Program~\eqref{eq:progstable} can be rewritten as follows.
	\begin{equation} \tag{SVM}
	\label{eq:SVM}
	\begin{split}
	& \min_{\alpha\in \R^N} \sum_{i\in [N]} \left(1-y_i \sum_{j\in [N]} \alpha_j k(x_j,x_i)\right)_+ + \lambda \|\sum_{i,j\in [N]}\alpha_i \alpha_j k(x_i,x_j)\|_k^2 \quad s.t. \\
	& ~ \Omega(f)\leq 0.
	\end{split}
	\end{equation}
	This model has been considered in~\cite{zafar2017fairness,zafar2017fair} that aims to avoid disparate impact/disparate mistreatment.
	Applying Theorems~\ref{thm:stable} and~\ref{thm:stable2}, and the fact that $L(\cdot,\cdot)$ is 1-admissible (Remark~\ref{remark:assumption}), we directly have the following corollary.
	
	\begin{corollary}
		\label{cor:SVM}
		Suppose the learning algorithm $\calA$ computes a minimizer $\calA_S$ of Program~\eqref{eq:SVM}.
		\begin{itemize}
			\item If $k(x_i,x_i)\leq \kappa^2<\infty$ for each $i\in [N]$, then $\calA$ is $\frac{\kappa^2}{\lambda N}$-uniformly stable.
			\item Let $G = \sup_{f=\alpha^\top \phi(\cdot)\in \calF: \Omega(f)\leq 0} \sup_{s\in \calD} \|\nabla_\alpha L(f,s)\|_2$.
			Then $\calA$ is $\frac{G^2}{\lambda N}$-uniformly stable.
		\end{itemize} 
	\end{corollary}
	
	\paragraph{Least square regression.}
	The only difference from soft margin SVM is the loss function, which is defined as follows.
	\[
	L(f,s) = (f(x)-y)^2.
	\]
	Then Program~\eqref{eq:progstable} can be rewritten as follows.
	\begin{equation} \tag{LS}
	\label{eq:leastsquare}
	\begin{split}
	& \min_{\alpha\in \R^N} \sum_{i\in [N]} \left(y_i-\sum_{j\in [N]} \alpha_j k(x_j,x_i)\right)^2 + \lambda \|\sum_{i,j\in [N]}\alpha_i \alpha_j k(x_i,x_j)\|_k^2 \quad s.t. \\
	& ~ \Omega(f)\leq 0.
	\end{split}
	\end{equation}
	Applying Theorems~\ref{thm:stable} and~\ref{thm:stable2}, we have the following corollary.
	
	\begin{corollary}
		\label{cor:leastsquare}
		Suppose the learning algorithm $\calA$ computes a minimizer $\calA_S$ of Program~\eqref{eq:leastsquare}.
		\begin{itemize}
			\item If $B=\max_{x\in \calX} |f(x)|$ and $k(x_i,x_i)\leq \kappa^2<\infty$ for each $i\in [N]$, then $\calA$ is $\frac{(2B+2)^2\kappa^2}{\lambda N}$-uniformly stable.
			\item Let $G = \sup_{f=\alpha^\top \phi(\cdot)\in \calF: \Omega(f)\leq 0} \sup_{s\in \calD} \|\nabla_\alpha L(f,s)\|_2$.
			Then $\calA$ is $\frac{G^2}{\lambda N}$-uniformly stable.
		\end{itemize}
	\end{corollary}
	
	\begin{proof}
		We only need to verify that $L(\cdot,\cdot)$ is $(2B+2)$-admissible.
		For any $x,x'\in \calX$ and $y\in \left\{-1,1\right\}$, we have
		\begin{align*}
		&\left|(f(x)-y)^2-(f(x')-y)^2\right|\\
		=&\left|(f(x)-f(x'))\cdot (f(x)+f(x')-2y)\right| \\
		\leq& (|f(x)|+|f(x')|+2) \cdot\left|(f(x)-f(x'))\right| \\
		\leq& (2B+2)\left|(f(x)-f(x'))\right|.
		\end{align*}
		This completes the proof.
	\end{proof}

	\paragraph{Logistic regression.}
	Again, the only difference from soft margin SVM is the loss function, which is defined as follows.
	\[
	L(f,s) = \ln (1+e^{-y f(x)}).
	\]
	This model has been widely used in the literature~\cite{zafar2017fairness,zafar2017fair,goel2018non}.
	Then Program~\eqref{eq:progstable} can be rewritten as follows.
	\begin{equation} \tag{LR}
	\label{eq:logistic}
	\begin{split}
	& \min_{\alpha\in \R^N} \sum_{i\in [N]} \ln \left(1+y_i\cdot e^{-\sum_{j\in [N]} \alpha_j k(x_j,x_i)}\right) + \lambda \|\sum_{i,j\in [N]}\alpha_i \alpha_j k(x_i,x_j)\|_k^2 \quad s.t. \\
	& ~ \Omega(f)\leq 0.
	\end{split}
	\end{equation}
	Applying Theorem~\ref{thm:stable} and~\ref{thm:stable2}, and the fact that $L(\cdot,\cdot)$ is 1-admissible (Remark~\ref{remark:assumption}), we have the following corollary.
	
	\begin{corollary}
		\label{cor:logisitc}
		Suppose the learning algorithm $\calA$ computes a minimizer $\calA_S$ of Program~\eqref{eq:logistic}.
		\begin{itemize}
			\item If $k(x_i,x_i)\leq \kappa^2<\infty$ for each $i\in [N]$, then $\calA$ is $\frac{\kappa^2}{\lambda N}$-uniformly stable.
			\item Let $G = \sup_{f=\alpha^\top \phi(\cdot)\in \calF: \Omega(f)\leq 0} \sup_{s\in \calD} \|\nabla_\alpha L(f,s)\|_2$.
			Then $\calA$ is $\frac{G^2}{\lambda N}$-uniformly stable.
		\end{itemize} 
	\end{corollary}
	
	\section{Proof of Theorem~\ref{thm:stable}}
	\label{sec:main_proof}
	
	It remains to prove the main result -- Theorem~\ref{thm:stable}.
	For convenience, we define $g=\calA_S$ and $g^{i}=\calA_{S^{i}}$.
	We first give Lemma~\ref{lm:bound} for preparation. 
	This lemma is the one of the places differences from the argument in~\cite{bousquet2002stability} since our framework includes a fairness constraint.
	To prove the lemma, we use the fact that $|g-g^i|_k^2$ is equivalent to the Bregman divergence between $g$ and $g^i$. 
	Then by the fact that $\Omega(f)$ is convex, we can upper bound the Bregman divergence by the right side of the inequality.
	Combining Lemma~\ref{lm:bound} and Claim~\ref{cl:RKHS}, we can upper bound $|g(x)-g^i(x)|$ for any $x\in \calX$. 
	This implies a uniform stability guarantee by the assumption that $L(\cdot,\cdot)$ is $\sigma$-admissible.

	\begin{lemma}
		\label{lm:bound}
		For any $i\in [N]$, we have
		\[
		\|g-g^{i}\|_k^2 \leq \frac{\sigma}{2\lambda N} \left(|g(x_i)-g^{i}(x_i)|+|g(x'_i)-g^{i}(x'_i)| \right).
		\]
	\end{lemma}

	\begin{proof}
		Given a differentiable convex function $F: \calF\times \calF\rightarrow \R$, we define the Bregman divergence by
		\[
		d_F(f,f') = F(f)-F(f')-\langle f-f',\nabla F(f') \rangle, ~\forall f,f'\in \calF.
		\]
		Define $R:\calF\rightarrow \R$ by
		\[
		R(f):= \frac{1}{N}\sum_{i\in [N]} L(f,s_i) + \lambda \|f\|_k^2,~\forall f\in \calF.
		\]
		Also define $R^{i}:\calF\rightarrow \R$ by
		\[
		R^{i}(f):= \frac{1}{N} \left(\sum_{j\neq i} L(f,s_j) + L(f,s'_i) \right) + \lambda \|f\|_k^2,~\forall f\in \calF.
		\]
		By definition of $g$ and $g^{i}$, we have
		\begin{eqnarray}
		\label{ineq:1}
		\begin{split}
		d_{R}(g^{i},g) = R(g^{i})-R(g)-\langle g^{i}-g,\nabla R(g) \rangle
		\leq R(g^{i})-R(g),
		\end{split}
		\end{eqnarray}
		and
		\begin{eqnarray}
		\label{ineq:2}
		\begin{split}
		d_{R^{i}}(g,g^{i}) = R^{i}(g)- R^{i}(g^{i})-\langle g-g^{i},\nabla  R^{i}(g^{i}) \rangle
		\leq  R^{i}(g)- R^{i}(g^{i}).
		\end{split}
		\end{eqnarray}
		By Inequalities~\eqref{ineq:1} and~\eqref{ineq:2}, we have
		\begin{eqnarray}
		\label{ineq:3}
		\begin{split}
		&d_{R}(g^{i},g) + d_{R^{i}}(g,g^{i})\\
		\leq & R(g^{i})-R(g)+R^{i}(g)- R^{i}(g^{i}) \\
		=& \frac{1}{N} \left(L(g^{i},s_i)-L(g,s_i) + L(g, s'_i)- L(g^i, s'_i) \right).
		\end{split}
		\end{eqnarray}
		Since $d_{A+B}=d_A+d_B$, we have
		\begin{eqnarray}
		\label{ineq:4}
		\begin{split}
		& 2\lambda \|g-g^{i}\|_k^2 &\\
		=& \lambda d_{\|\cdot\|_k^2}(g,g^{i}) + \lambda d_{\|\cdot\|_k^2}(g^{i},g) 
		& (\text{Defn. of $\|\cdot\|_k^2$}) \\
		= & d_{R^{i}}(g,g^{i}) - d_{\sum_{j\neq i}L(\cdot,s_j)} (g,g^{i})
		+ d_{R}(g^{i},g) - d_{\sum_{i}L(\cdot,s_i)} (g^{i},g) 
		& (d_{A+B}=d_A+d_B) \\
		\leq& d_{R^{i}}(g,g^{i}) + d_{R}(g^{i},g) 
		& (\text{nonnegativity of $d_F$}) \\
		\leq& \frac{1}{N} \left(L(g^{i},s_i)-L(g,s_i) + L(g, s'_i)- L(g^i, s'_i) \right) &\\
		\leq& \frac{\sigma }{N}\left(|g(x_i)-g^{i}(x_i)|+|g(x'_i)-g^{i}(x'_i)| \right). 
		& (L(\cdot,\cdot) \text{ is $\sigma$-admissible} )
		\end{split}
		\end{eqnarray}
		This completes the proof.
	\end{proof}
	
	\noindent
	Lemma~\ref{lm:bound} upper bounds $\|g-g^i|_k^2$.
	Then combining Lemma~\ref{lm:bound} and Claim~\ref{cl:RKHS}, we can upper bound $|g(x)-g^i(x)|$ for any $x\in \calX$. 
	This implies a uniform stability guarantee by the assumption that $L(\cdot,\cdot)$ is $\sigma$-admissible.
	\begin{proof}[Proof of Theorem~\ref{thm:stable}]
		By Claim~\ref{cl:RKHS}, we have
		\[
		|g(x_i)-g^{i}(x_i)|\leq \|g-g^{i}\|_k \sqrt{k(x_i,x_i)}\leq \kappa \|g-g^{i}\|_k,
		\]
		\[
		|g(x'_i)-g^{i}(x'_i)|\leq \|g-g^{i}\|_k \sqrt{k(x'_i,x'_i)}\leq \kappa \|g-g^{i}\|_k.
		\]
		Combining the above inequalities with Lemma~\ref{lm:bound}, we have
		$\|g-g^{i}\|_k\leq \frac{\sigma \kappa}{\lambda N}$.
		Hence, for any sample $s=(x,z,y)\in \calD$, we have
		$
		|g(x)-g^{i}(x)|\leq \kappa \|g-g^{i}\|_k\leq \frac{\sigma \kappa^2}{\lambda N}.
		$
		Moreover, since $L(\cdot,\cdot)$ is $\sigma$-admissible, we have
		\[
		|L(g,s)-L(g^{i},s)|\leq \sigma |g(x)-g^{i}(x)|\leq \frac{\sigma^2 \kappa^2}{\lambda N}.
		\]
		By definitions of $g$ and $g^{i}$, the above inequality completes the proof of stability guarantee.

		For the the increase of the empirical risk, let $F(f):=\frac{1}{N} \sum_{i\in [N]} L(f,s_i) + \lambda \|f\|_k^2$ for any $f\in \calF$.
		By Theorem 8 of~\cite{shalev2010learnability}, we have the following claim: for any classifier $h\in \calF$ satisfying that $\Omega(h)\leq 0$, $\Exp_{S\sim \Im^N} \left[ F(g) - F(h) \right]$ is consistent with the uniform stability guarantee of $\calA$, i.e.,
		\begin{eqnarray}
		\label{eq:stable}
		\Exp_{S\sim \Im^N} \left[ F(g) - F(h) \right] \leq \frac{\sigma^2 \kappa^2}{\lambda N}.
		\end{eqnarray}
		Let $h=f^\star$, we have
		\begin{eqnarray*}
			\begin{split}
				&\Exp_{S\sim \Im^N} \left[R(\calA_S) \right] - \Exp_{s\sim \Im} \left[L(f^\star,s)\right] &\\
				= & \Exp_{S\sim \Im^N} \left[R(\calA_S) - \frac{1}{N} \sum_{i\in [N]} L(f^\star, s_i) \right] & \\
				= & \Exp_{S\sim \Im^N} \left[F(g)-\lambda\|g\|_k^2-F(f^\star)+ \lambda\|f^\star\|_k^2\right] 
				& (\text{Defns. of $g$ and $F(\cdot)$}) \\
				\leq &\Exp_{S\sim \Im^N} \left[F(g)-F(f^\star)\right] + \lambda\|f^\star\|_k^2 & (\|g\|_k^2\geq 0) \\
				\leq &\frac{\sigma^2 \kappa^2}{\lambda N} + \lambda B^2 & (\text{Ineq.~\eqref{eq:stable} and } \|f^\star\|_k\leq B).
			\end{split}
		\end{eqnarray*}
		This completes the proof.
	\end{proof}

	\begin{table*}[t]
		\centering
		\caption{The performance (mean and standard deviation in parenthesis), of \textbf{KAAS-St} and \textbf{ZVRG-St} with respect to accuracy and the fairness metrics  $\gamma$ on the \textbf{Adult} dataset with race/sex attribute.
		} 
		{\fontsize{7}{8} \selectfont
			\begin{tabular}{|*{9}{c|}} \hline
				%\multicolumn{12}{|c|}{This paper}  \tabularnewline \hline
				& & & \multicolumn{6}{|c|}{$\lambda$}  \tabularnewline \cline{4-9}			
				& & & 0 & 0.01 & 0.02 & 0.03 & 0.04 & 0.05  \tabularnewline \hline
				\multirow{4}{*}{\textbf{ZVRG-St}} & \multirow{2}{*}{Race} & Acc. & 0.844(0.001) & 0.842(0.001) & 0.841(0.001) & 0.840(0.001) & 0.838(0.001) & 0.838(0.001) \tabularnewline \cline{3-9}
				& & $\gamma$  & 0.577(0.031) & 0.667(0.020) & 0.686(0.015) & 0.711(0.016) & 0.743(0.013) & 0.761(0.012) \tabularnewline  \cline{2-9}
				& \multirow{2}{*}{Sex} & Acc. & 0.844(0.001) & 0.840(0.001) & 0.838(0.001) & 0.838(0.001) & 0.837(0.001) & 0.836(0.001) \tabularnewline \cline{3-9}
				& & $\gamma$  & 0.331(0.041) & 0.501(0.011) & 0.495(0.009) & 0.478(0.009) & 0.463(0.009) & 0.469(0.009) \tabularnewline  \hline
				\multirow{4}{*}{\textbf{KAAS-St}} & \multirow{2}{*}{Race} & Acc. & 0.850(0.001) & 0.844(0.001) & 0.843(0.001) & 0.839(0.001) & 0.837(0.001) & 0.835(0.001) \tabularnewline \cline{3-9}
				& & $\gamma$  & 0.571(0.019) & 0.359(0.024) & 0.302(0.011) & 0.301(0.011) & 0.300(0.015) & 0.298(0.015) \tabularnewline  \cline{2-9}
				& \multirow{2}{*}{Sex} & Acc. & 0.850(0.002) & 0.848(0.001) & 0.844(0.001) & 0.839(0.001) & 0.837(0.001) & 0.835(0.001) \tabularnewline \cline{3-9}
				& & $\gamma$ & 0.266(0.011) & 0.226(0.011) & 0.165(0.008) & 0.136(0.007) & 0.128(0.006) & 0.128(0.005) \tabularnewline  \hline
				\multirow{4}{*}{\textbf{GYF-St}} & \multirow{2}{*}{Race} & Acc. & 0.849(0.001) & 0.845(0.001) & 0.844(0.001) & 0.842(0.001) & 0.840(0.001) & 0.835(0.001) \tabularnewline \cline{3-9}
				& & $\gamma$  & 0.558(0.020) & 0.679(0.013) & 0.690(0.017) & 0.710(0.018) & 0.740(0.014) & 0.753(0.013) \tabularnewline  \cline{2-9}
				& \multirow{2}{*}{Sex} & Acc. & 0.850(0.002) & 0.845(0.001) & 0.844(0.001) & 0.842(0.001) & 0.840(0.001) & 0.839(0.001) \tabularnewline \cline{3-9}
				& & $\gamma$ & 0.275(0.010) & 0.245(0.004) & 0.242(0.004) & 0.241(0.005) & 0.245(0.005) & 0.234(0.008) \tabularnewline  \hline
			\end{tabular}
		}
		
		\label{tab:acc}
	\end{table*}

\section{Empirical results}
\label{sec:experiment}

\subsection{Empirical setting}

\paragraph{Algorithms and baselines.} 
We select three fair classification algorithms designed to ensure statistical parity that can be  formulated in the  convex optimization framework of Program~\eqref{eq:progcon}.
We choose \textbf{ZVRG}~\cite{zafar2017fair} since it is reported to achieve the better fairness than, and comparable accuracy to, other algorithms~\cite{friedler2018comparative}.
We also select
\textbf{KAAS}~\cite{kamishima2012fairness} and \textbf{GYF}~\cite{goel2018non} as representatives of algorithms that are formulated as Program~\eqref{eq:progreg}.
Specifically, \cite{goel2018non} showed that the performance of \textbf{GYF} is comparable to \textbf{ZVRG} over the \textbf{Adult} dataset. 
We extend them by introducing a stability-focused regularization term.\footnote{The codes are available on \url{https://github.com/huanglx12/Stable-Fair-Classification}.}
%
% We extend two existing algorithms, which provides fair classification results with respect to statistical parity.
%
\vspace{-3mm}
\begin{itemize}
	\item \textbf{ZVRG~\cite{zafar2017fairness}.} Zafar et al. re-express fairness constraints (which can be nonconvex) via a convex relaxation. 
	This allows them to maximize accuracy subject to fairness constraints.\footnote{There exists a threshold parameter in the constraints. In this paper, we set the parameter to be default 0.1.}
	We denote  the extended, stability included, algorithm by \textbf{ZVRG-St}.
	\vspace{-3mm}
	
	\item \textbf{KAAS~\cite{kamishima2012fairness}.} Kamishima et al. introduce a fairness-focused regularization term and apply it to a logistic regression classifier. 
	Their approach requires numerical input and a binary sensitive attribute.
	Let the extended algorithm be \textbf{KAAS-St}.
	\vspace{-3mm}
	
	\item \textbf{GYF~\cite{goel2018non}.} Goel et al. introduce \emph{negative weighted sum of logs} as fairness-focused regularization term and apply it to a logistic regression classifier. 
	Let the extended algorithm be \textbf{GYF-St}.
\end{itemize}

\begin{figure*}[htp]
	\centering
	\begin{minipage}[b]{.45\textwidth}
		%		\begin{minipage}{.47\textwidth}
		\includegraphics[width=1\linewidth]{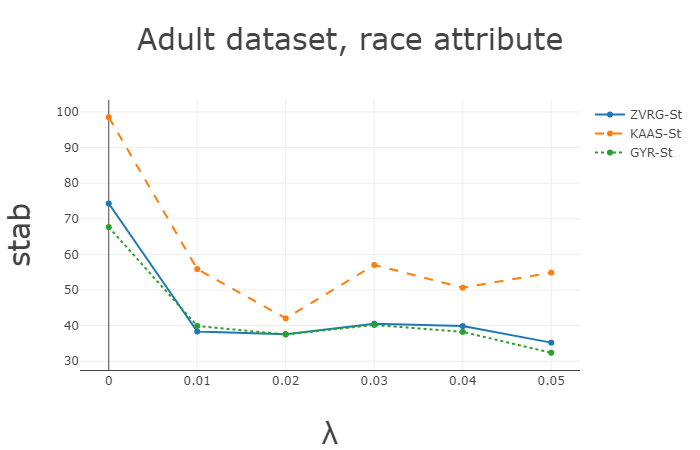}
		\caption{$\mathrm{stab}$ vs. $\lambda$ for race attribute.}
		\label{fig:race}
	\end{minipage}\qquad
	%	\end{minipage}\hspace{0.05cm}
	\begin{minipage}[b]{.45\textwidth}
		%	\centering
		\includegraphics[width=1\linewidth]{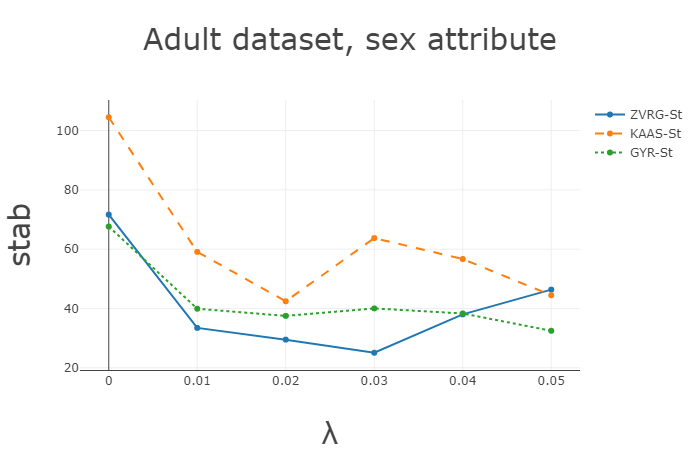}
		\caption{$\mathrm{stab}$ vs. $\lambda$ for sex attribute.}
		\label{fig:sex}
	\end{minipage}
\end{figure*}

\vspace{-3mm}

\paragraph{Dataset.} Our simulations are over an income dataset \textbf{Adult} ~\cite{Dua:2017}, that records the demographics of 45222 individuals,
along with a binary label indicating whether the income of an individual is greater than 50k USD or not. 
{We use the pre-processed dataset as in~\cite{friedler2018comparative}.} 
We take race and sex to be the sensitive attributes, that are binary in the dataset.

\vspace{-3mm}

\paragraph{Stability metrics.} 
The following stability metric that measures the prediction difference between classifiers learnt from two random training sets.
Given an integer $N$, a testing set $T$ and algorithm $\calA$, we define $\mathrm{stab}_T(\calA)$ as $\\$
%
%	\begin{eqnarray*}
%	{\scriptstyle	\begin{split}
%		& \\
$|T|\cdot \Pr_{S,S'\sim \Im^N,X\sim T,\calA} \left[
\I\left[\calA_S(X)\geq 0\right]\neq \I\left[\calA_{S'}(X)\geq 0\right]  \right].$
%	\end{split}}
%	\end{eqnarray*}
%

$\mathrm{stab}_T(\calA)$ indicates the expected number of different predictions of $\calA_S$ and $\calA_{S'}$ over the testing set $T$.
Note that this metric is considered in~\cite{friedler2018comparative}, but is slightly different from prediction stability since $S$ and $S'$ may differ by more than one training sample.
%
% Since \textbf{Adult} is of large size, changing only one sample barely affects the learned classifier.
%
% Also in practice, there may exist noises on multiple training samples.
% 
We investigate $\mathrm{stab}_T(\calA)$ instead of prediction stability so that we can distinguish the performances of prediction difference under different regularization parameters.
% Note that $\left| \I\left[\calA_S(x)\geq 0\right]- \I\left[\calA_{S'}(x)\geq 0\right] \right|=1$ if and only if the predictions of $\calA_S$ and $\calA_{S'}$ over sample $s$ are different.
% 
Since $\Im$ is unknown, we generate $n$ training sets $S_1,\ldots,S_{n}$ and use the following metric to estimate $\mathrm{stab}_T(\calA)$:
\begin{eqnarray}
\label{eq:stability}
\begin{split}
\mathrm{stab}_{T,n}(\calA):=& \frac{1}{n(n-1)} \sum_{i,j\in [n]: i\neq j}  \sum_{s=(x,z,y)\in T}\\
&\left| \I\left[\calA_{S_i}(x)\geq 0\right]- \I\left[\calA_{S_j}(x)\geq 0\right] \right|.
\end{split}
\end{eqnarray}

Note that we have $\Exp_{S_1,\ldots,S_n}\left[\mathrm{stab}_{T,n}(\calA)\right]=\mathrm{stab}_{T}(\calA)$.
%
% Thus, Algorithm $\calA$ is stable if $\mathrm{stab}_{T,n}(\calA)$ is small.
%

\vspace{-3mm}

\paragraph{Fairness metric.} Let $D$ denote the empirical distribution over the testing set.
Given a classifier $f$, we consider a fairness metric for statistical rate, which has been applied in~\cite{menon2018the,agarwal2018reductions}.
Suppose the sensitive attribute is binary, i.e., $Z\in \left\{0,1\right\}$.
\begin{eqnarray}
\label{eq:gamma}
\begin{split}
&\gamma(f):=\\
&\min\left\{\frac{\Pr_D\left[f=1\mid Z=0\right]}{\Pr_D\left[f=1\mid Z=1\right]},\frac{\Pr_D\left[f=1\mid Z=1\right]}{\Pr_D\left[f=1\mid Z=0\right]} \right\}.
\end{split}
\end{eqnarray}

\noindent
Our framework can be easily extended to other fairness metrics; see a summary in Table 1 of~\cite{celis2018classification}.

\paragraph{Implementation details.} We first generate a common testing set (20$\%$). 
Then we perform 50 repetitions, in which we uniformly sample a training set (75$\%$) from the remaining data.
For all three algorithms, we set the regularization parameter $\lambda$ to be $0, 0.01, 0.02, 0.03, 0.04, 0.05$ and compute the resulting stability metric $\mathrm{stab}$, average accuracy and average fairness.
Note that $\lambda = 0$ is equivalent to the case without stability-focused regularization term.

\subsection{Results}

Our simulations indicate that introducing a stability-focused regularization term can make the algorithm more stable by slightly sacrificing accuracy.
Table~\ref{tab:acc} summarizes the accuracy and fairness metric under different regularization parameters $\lambda$.
As $\lambda$ increases, the average accuracy slightly decreases, by at most 1.5$\%$, for all algorithms including \textbf{ZVRG-St}, \textbf{KAAS-St} and \textbf{GYR-St}.
As for the fairness metric, as $\lambda$ increases, the mean of $\gamma$ decreases for \textbf{KAAS-St} and increases for \textbf{ZVRG-St} for both race and sex attribute.
For \textbf{GYF-St}, the performance of fairness metric depends on the sensitive attribute: as $\lambda$ increases, the mean of $\gamma$ decreases for the sex attribute and increases for the race attribute.
Note that the fairness metric $\gamma$ of \textbf{KAAS-St} and \textbf{GYF-St} is usually smaller than that of \textbf{ZVRG-St} with the same $\lambda$.
The results indicate that \textbf{ZVRG-St} achieves the better fairness than, and comparable accuracy to, other algorithms.
Another observation is that the standard deviation of $\gamma$ decreases by introducing the regularization term. 
Specifically, considering the sex attribute, the standard deviation of $\gamma$ is 4.1$\%$ when $\lambda=0$ and decreases to about 1$\%$ by introducing a stability-focused regularization term. 
This observation implies that our extended framework improves the stability.

Figures~\ref{fig:race} and~\ref{fig:sex} summarize the stability metrics $\mathrm{stab}$ under different regularization parameters $\lambda$.
By introducing stability-focused regularization term, $\mathrm{stab}$ indeed decreases for both race and sex attributes.
Observe that $\mathrm{stab}$ can decrease by a half by introducing the regularization term for all three algorithms. 
Note that $\mathrm{stab}$ of \textbf{KAAS-St} is always larger than that of \textbf{ZVRG-St} and \textbf{GYF-St} with the same $\lambda$.
The stability of \textbf{ZVRG-St} and \textbf{GYF-St} is comparable.
Interestingly, $\mathrm{stab}$ does not monotonically decrease as $\lambda$ increases due to the fairness requirements.
The reason might be as follows: as $\lambda$ increases, the model parameters of the learned classifiers should decrease monotonically. 
However, it is possible that a classifier with smaller model parameters is more sensitive to random training sets. 
In this case, if the effect of $\lambda$ to $\mathrm{stab}$ is less when compared to the effect of model parameters, $\mathrm{stab}$ might not decrease monotonically with $\lambda$. 
Hence, selecting a suitable regularization parameter $\lambda$ is valuable in practice, e.g., considering \textbf{ZVRG-St} for sex attribute, letting $\lambda=0.03$ achieves better performance of accuracy, fairness and stability than letting $\lambda = 0.05$.

		\section{Conclusion and future directions}
	\label{sec:conclusion}
	
	We propose an extended framework for fair classification algorithms that are formulated as optimization problems.
	Our framework comes with a stability guarantee and we also provide an analysis of the resulting accuracy.
	The analysis can be used to inform the selection of the regularization parameter.
	The empirical results show that our framework indeed improves stability by slightly sacrificing the accuracy.
	%
	
	%Our work leaves several interesting future directions. 
	%
	There exist other fair classification algorithms that are not formulated as optimization problems, e.g., shifting the decision boundary of a baseline classifier~\cite{fish2016confidence,hardt2016equality} or pre-processing the training data~\cite{feldman2015certifying,krasanakis2018adaptive}.
	It is interesting to investigate and improve the stability guarantee of those algorithms.
	Another potential direction is to combine stability and fairness for other automated decision-making tasks, e.g., ranking \cite{celis2018ranking,yang2017measuring}, summarization \cite{celis2018fair}, personalization \cite{celis2018an,celis2019algorithmic}, multiwinner voting \cite{celis2018multiwinner}, and online advertising \cite{celis2019online}.

	\bibliography{references}
	\bibliographystyle{plain}
	
	\appendix
	
	\section{Proof of Theorem~\ref{thm:stable2}}
	\label{sec:proof2}
	
	\begin{proof}
		By Inequality~\eqref{ineq:4} in the proof of Lemma~\ref{lm:bound}, we have
		\begin{eqnarray}
		\label{ineq:2_1}
		\begin{split}
		& 2\lambda \|v-v^i\|_2^2
		\leq \frac{1}{N} \left(L(g^{i},s_i)-L(g,s_i) + L(g, s'_i)- L(g^i, s'_i) \right).
		\end{split}
		\end{eqnarray}
		Moreover, we have for any $f=\alpha\cdot \phi(\cdot),f'=\alpha'\cdot \phi(\cdot)\in \calF$ and $s\in \calD$,
		\begin{eqnarray}
		\label{ineq:2_2}
		\begin{split}
		L(f,s)-L(f',s) \leq & \langle \nabla_\alpha L(f,s) , \alpha-\alpha' \rangle 
		&\text{(Convexity of $L(\cdot,s)$)} \\
		\leq & \|\nabla_\alpha L(\alpha,s)\|_2 \cdot \|\alpha-\alpha'\|_2 & \\
		\leq & G\|\alpha-\alpha'\|_2 
		& \text{(Defn. of $G$)}.
		\end{split}
		\end{eqnarray}
		Combining with Inequalities~\eqref{ineq:2_1} and~\eqref{ineq:2}, we have
		\begin{eqnarray*}
			\begin{split}
				\|v-v^i\|_2^2 
				\leq & \frac{1}{2\lambda N} \left(L(g^{i},s_i)-L(g,s_i) + L(g, s'_i)- L(g^i, s'_i) \right) 
				&\text{(Ineq.~\eqref{ineq:2_1})} \\
				\leq & \frac{1}{2\lambda N} \left(G\|v-v^i\|_2 +G\|v-v^i\|_2\right) 
				&\text{(Ineq.~\eqref{ineq:2_2})} \\
				=&\frac{G}{\lambda N} \|v-v^i\|_2.&
			\end{split}
		\end{eqnarray*}
		It implies that $\|v-v^i\|_2\leq \frac{G}{\lambda N}$. 
		Combining with Inequality~\eqref{ineq:2}, we have for any $s\in \calD$,
		\[
		L(g,s)-L(g^i,s) \leq G\|v-v^i\|_2\leq \frac{G^2}{\lambda N}.
		\]
		This completes the proof for the stability guarantee.
		For the sacrifice in the empirical risk, the argument is the same as that of Theorem~\ref{thm:stable}.
	\end{proof}

	\section{Details of Remark~\ref{remark:assumption}}
	\label{sec:dis}
	
	\begin{itemize}
		\item Prediction error: $f(x)\in \left\{-1,1\right\}$ for any pair $(f,x)$ and $L(f(x),y) = \I\left[f(x)\neq y\right] $,\footnote{Here, $\I\left[\cdot\right]$ is the indicator function.} then we have that 
		\begin{eqnarray*}
			& & \left|L(f(x),y)-L(f(x'),y)\right| \\
			&=& \left|\I\left[f(x)\neq y\right]-\I\left[f(x')\neq y\right] \right| \\
			&=& \I\left[f(x)\neq f(x')\right]= \frac{1}{2} \left|f(x)-f(x') \right|,
		\end{eqnarray*}
		which is $\frac{1}{2}$-admissible.
		\item Soft margin SVM: $L(f,s) = (1-yf(x))_+$,\footnote{$(a)_+ = a$ if $a\geq 0$ and otherwise $(a)_+ = 0$.} then we have that
		\begin{eqnarray*}
			& &\left|L(f(x),y)-L(f(x'),y) \right| \\
			&=& \left|(1-yf(x))_+ -(1-yf(x'))_+ \right| \\
			&\leq& \left| yf(x)-yf(x') \right| \\
			&=& \left|f(x)-f(x')\right|,
		\end{eqnarray*}
		which is 1-admissible. 
		\item Least Squares regression: $L(f,s) = (f(x)-y)^2$. Suppose $f(x)\in [-1,1]$ for any $x\in \calX$, then we have that 
		\begin{eqnarray*}
			& & \left|L(f(x),y)-L(f(x'),y) \right| \\
			&=&  \left| (f(x)-y)^2 - (f(x')-y)^2 \right| \\
			&=& \left|(f(x)+f(x')-2y)(f(x)-f(x'))\right|\\
			&\leq& 4 \left|f(x)-f(x')\right|,
		\end{eqnarray*}
		which is 4-admissible.
		\item Logistic regression: $L(f,s) = \ln (1+e^{-y f(x)})$. 
		Note that we have for any $x\in \calX$ and $y\in \left\{-1,1\right\}$,
		\begin{eqnarray*}
			& &\left|\nabla_{f(x)} \ln (1+e^{-y f(x)}) \right| \\
			&=& \left|\frac{-y e^{-y f(x)}}{1+e^{-y f(x)}} \right| 
			= \left|\frac{e^{-y f(x)}}{1+e^{-y f(x)}} \right|
			\leq 1.
		\end{eqnarray*}
		Hence, the loss function $L(f,s) = \ln (1+e^{-y f(x)})$ is 1-admissible.
	\end{itemize}

\end{document}